\documentclass{article}

\usepackage[utf8]{inputenc} 
\usepackage[T1]{fontenc}    
\usepackage{hyperref}       
\usepackage{url}            
\usepackage{booktabs}       
\usepackage{amsfonts}       
\usepackage{nicefrac}       
\usepackage{microtype}      
\usepackage{xcolor}         
\usepackage{epsfig}
\usepackage[ruled,linesnumbered]{algorithm2e}
\usepackage{algpseudocode}
\usepackage{amsmath,amssymb,amsthm}
\usepackage{xspace}
\usepackage{enumitem}
\usepackage{bbm}
\usepackage{graphicx}
\usepackage{cases}
\usepackage{authblk}

\newcommand{\scheme}{SecFC\xspace}
\newcommand{\kfed}{$k$-FED\xspace}

\let\emptyset\varnothing

\newtheorem{theorem}{Theorem}

\theoremstyle{remark}
\newtheorem{remark}{\bf{Remark}}

\theoremstyle{definition}

\title{Secure Federated Clustering}

	\author[1,2]{Songze Li}
	\author[1]{Sizai Hou}
	\author[3]{Baturalp Buyukates}
	\author[3]{Salman Avestimehr}
	\affil[1]{IoT Thrust, The Hong Kong University of Science and Technology (Guangzhou)}
	\affil[2]{Department of Computer Science and Engineering, The Hong Kong University of Science and Technology}
	\affil[3]{ Department of Electrical and Computer Engineering, University of Southern California}

\date{}

\begin{document}

\maketitle

\begin{abstract}
We consider a foundational unsupervised learning task of $k$-means data clustering, in a federated learning (FL) setting consisting of a central server and many distributed clients. We develop SecFC, which is a secure federated clustering algorithm that simultaneously achieves 1) \emph{universal performance:} no performance loss compared with clustering over centralized data, regardless of data distribution across clients; 2) \emph{data privacy:} each client's private data and the cluster centers are not leaked to other clients and the server. In SecFC, the clients perform Lagrange encoding on their local data and share the coded data in an information-theoretically private manner; then leveraging the algebraic structure of the coding, the FL network exactly executes the Lloyd's $k$-means heuristic over the coded data to obtain the final clustering. Experiment results on synthetic and real datasets demonstrate the universally superior performance of SecFC for different data distributions across clients, and its computational practicality for various combinations of system parameters. Finally, we propose an extension of SecFC to further provide membership privacy for all data points. 
\end{abstract}

\section{Introduction}
Federated learning (FL) is an increasingly popular distributed learning paradigm, which enables decentralized clients to collaborate on training an AI model over their private data collectively, without surrendering the private data to a central cloud. When running the FedAvg algorithm introduced in~\cite{mcmahan2017communication}, 
each client trains a local model with its private data, and sends the local model to a central server for further aggregation into a global model. Since its introduction, FL has demonstrated great potential in improving the performance and security of a wide range of learning problems, including next-word prediction~\cite{yang2018applied}, recommender systems~\cite{jalalirad2019simple}, and predictive models for healthcare~\cite{rieke2020future}. 

In this work, we focus on a foundational unsupervised learning task - clustering. In an FL system, the clustering problem can be categorized into client clustering and data clustering. Motivated by heterogeneous data distribution across clients and the needs for training personalized models~\cite{t2020personalized,fallah2020personalized,li2021ditto}, client clustering algorithms are developed to partition the clients into different clusters to participate in training of different models~\cite{mansour2020three,ghosh2020efficient,ouyang2021clusterfl,kim2021dynamic}. On the other hand, data clustering problem aims to obtain a partition of data points distributed on FL clients, such that some clustering loss is minimized. In~\cite{chung2022federated}, a model-based approach UIFCA was proposed to train a generative model for each cluster, such that each data point can be classified into one of the clusters using the models. Another recent work~\cite{Dennis21a} considers the classical $k$-means clustering problem, and proposed an algorithm \kfed to approximately execute the Lloyd's heuristic~\cite{lloyd1982least} in the FL setting. \kfed is communication-efficient as it only requires one-shot communication from the clients to the server.

While being the state-of-the-art federated data clustering algorithms, both UIFCA and \kfed are observed to be limited in the following two aspects: 1) The clustering performance heavily depends on the data distribution across clients. Specifically, the clustering performance of \kfed is guaranteed only when the data points on each client are from at most $k' \leq \sqrt{k}$ clusters, and a non-empty cluster at each client has to contain a minimum number of points; 2) Leakage of data privacy. Just like other model-based FL algorithms, the data privacy of UIFCA is subject to model inversion attacks (see, e.g.,~\cite{zhu2019deep,geiping2020inverting,wang2019beyond,fowl2021robbing}). In \kfed, each client performs Lloyd's algorithm on local data, and sends the local cluster centers that are averages of private data points directly to the server, which allows a curious server to infer about the private dataset (e.g., number of local clusters, and the distribution of local data points).

Motivated by these limitations, we develop a secure federated clustering protocol named \scheme, for solving a $k$-means clustering problem over an FL network. \scheme builds upon the classical Lloyd's heuristic and has the following salient properties:
\begin{itemize}[leftmargin=*]
    \item {\bf Universal performance:} The clustering performance of \scheme matches exactly that of running Lloyd on the collection of all clients' data in a centralized manner, which is invariant to different data distributions at the clients;
    \item {\bf Information-theoretical data privacy:} Each client's private data is information-theoretically private from a certain number of colluding clients. The server does not know the values of the private data points and the cluster centers.
\end{itemize}

In the first phase of \scheme, each client locally generates some random noises and mixes them with its private data using Lagrange polynomial encoding~\cite{yu2019lagrange}, and then securely shares the coded data segments with the other clients. In the second phase, the clients and the server jointly execute the Lloyd's iterations over coded data. As each client obtains a coded version of the entire dataset in the first phase, leveraging the algebraic structure of the coding scheme, each client can compute a secret share of the distance between each pair of cluster center and data point. After receiving sufficient number of secret shares, the server is able to reveal the pair-wise distances between all cluster centers and all data points, so the center update can be carried out exactly as specified in the Lloyd's algorithm. Importantly, the server can only recover the distance between a cluster center and a data point, without knowing their respective locations, which is the key enabler for \scheme to achieve performance and privacy simultaneously. We further extend \scheme to strengthen it with membership privacy. Specifically, clients privately align the IDs of their local data using private set union~\cite{psu:1,psu:2,psu:3}, such that data clustering is performed without revealing which data point belongs to which client.


We conduct extensive experiments to cluster synthetic and real datasets on FL systems, and compare \scheme with the centralized Lloyd's algorithm and \kfed as our baselines. Under different data distributions across clients with different number of local clusters $k'$, \scheme consistently achieves almost identical performance as the centralized Lloyd, which outperforms \kfed. Also, unlike \kfed whose performance varies substantially with the value of $k'$, and deteriorates severely when $k'=k$, the performance of \scheme is almost invariant to the value of $k'$, indicating its universal superiority. We also evaluate the execution overheads of \scheme under various combinations of system parameters, which demonstrate the feasibility of applying \scheme to solve practical clustering tasks. 



\section{Background and Related Work}
Clustering is an extensively studied learning task. We discuss existing works in three categories.

\begin{itemize}[leftmargin=*]
\item \textbf{Centralized clustering.}
In centralized clustering, a central entity, i.e., server, performs the data partition. A well-known centralized clustering problem is $k$-means clustering, where the aim is to find $k$ centers such that the Euclidean distance from each data point to its closest center is minimized. Lloyd's heuristic \cite{lloyd1982least} is a popular iterative approach in performing $k$-means clustering due to its simplicity, even though its performance heavily depends on algorithm initialization. Unlike the Lloyd's heuristic which does not provide any provable guarantees, Kanungo et al.~\cite{Kanungo03} proposes an approximation algorithm for $k$-means clustering that is based on swapping cluster centers in and out, which yields a solution that is at most $(9+\epsilon)$ factor larger than the optimal solution. Awasthi and Sheffet \cite{Awasthi12} propose a variant of the Lloyd's heuristic that is equipped with such an approximation algorithm to guarantee convergence in polynomial time. 
\item \textbf{Parallel and distributed clustering.} Distributed and parallel clustering implementations are particularly useful in the presence of large datasets that are prominent in many learning applications. The bottleneck in the centralized $k$-means clustering is the computation of distances among the data points and the centers. In parallel implementations of $k$-means clustering, the idea is to split the dataset into different worker machines so that the distance computations are performed in parallel \cite{Dhillon99, Joshi03}. In addition to $k$-means clustering, there have been parallel implementations of other clustering schemes such as the density-based clustering scheme DBSCAN \cite{Xu99}. When the dataset is distributed in nature, it may not be practical to send local datasets to the central entity. In such cases, distributed clustering is favorable albeit a performance loss compared to the centralized setting \cite{Kargupta01, Januzaj03, Balcan13}.  
\item \textbf{Federated clustering.} Clustering techniques have been used in FL systems to handle data heterogeneity  \cite{Smith17, ghosh2020efficient, Sattler20, Dennis21a, Balakrishnan21}, particularly focusing on personalized model training and client selection. References \cite{Smith17, Sattler20, ghosh2020efficient} focus on clustering the clients to jointly train multiple models, one for each cluster, whereas reference \cite{Dennis21a} clusters the client data points to reach more personalized models and perform a more representative client selection at each iteration. Unlike the $k$-means clustering which uses a heuristic to generate clusters, model-based clustering assumes that the data to be clustered comes from a probabilistic model and aims to recover that underlying distribution \cite{Mukherjee19, Liu20}. In their recent work \cite{chung2022federated}, using this approach, Chung et al. cluster heterogeneous client data in a federated setting by associating each data point with a cluster (model) with the highest likelihood. As also highlighted by \cite{Balakrishnan21}, common to all these existing works on $k$-means clustering is the fact that none of them provides \emph{data privacy}, in the sense that cluster center information as well as clients' data are kept private. This motivates us to develop \scheme in this work, which is to the best of our knowledge, the first federated clustering algorithm with formal privacy guarantees.  
\end{itemize}

Before we close this section, we briefly review some prior works on secure clustering.

\textbf{Privacy-preserving clustering.} When the clustering is performed on private client data as in the case of federated clustering, the issue of secure clustering emerges. The aim here is to cluster the entire dataset without revealing anything to the clients, including all intermediate values, but the output, i.e., final partition of data.  Privacy-preserving clustering has been studied in the literature employing cryptographic techniques from secure multiparty computation, for the cases of horizontal, vertical, and arbitrary partitions of data \cite{Vaidya03, Jagannathan05, Bunn07, Patel12, Yuan17, Mohassel19}. Among these works, we want to highlight \cite{Mohassel19}, where authors implement a secret sharing-based secure $k$-means clustering scheme. Unlike the proposed \scheme scheme, however, \cite{Mohassel19} employs an additive secret sharing scheme and reveals the final cluster centers to the participants. In \scheme, clients use Lagrange coding to encode and share their private data. During the execution of
\scheme, neither final nor any intermediate candidate center information is disclosed to the clients or the server, which is not the case in other earlier works \cite{Vaidya03, Jagannathan05,Patel12} in the literature as well.

{\bf Notation.} For a positive integer $n$, let $[n]$ denote the set of integers $\{1,\ldots,n\}$. Let $|{\cal S}|$ denote the cardinality of a set ${\cal S}$.

\section{Preliminaries and Problem Formulation}\label{sec:formulation}

\subsection{$k$-means Clustering}
Given 
a dataset ${\bf X}$ consisting of $m$ data points ${\bf x}_1,\ldots,{\bf x}_m \in \mathbb{R}^d$, each with features of dimension $d$, the $k$-means clustering problem aims to obtain a $k$-partition of ${\bf X}$,
denoted by $ \mathcal{S} = \left( \mathcal{S}_1,\ldots,\mathcal{S}_k  \right) $, such that the following clustering loss is minimized. 
\begin{equation}\label{obj_fcn}
L(\mathcal{S}) 
= \sum_{h=1}^{k} \sum_{\mathbf{x}_i \in {\cal S}_{h}}\left\|\mathbf{x}_i-\boldsymbol{\mu}_{h}\right\|_2^{2},
\end{equation}
where $\boldsymbol{\mu}_{h}$ denotes the center of cluster ${\cal S}_h$, which is computed as $\boldsymbol{\mu}_{h} = \frac{1}{\left | \mathcal{S}_{h} \right | }\sum_{\mathbf{x}_i \in \mathcal{S}_{h}} {\mathbf{x}_i}$. 

A classical method for clustering is the Lloyd’s heuristic algorithm \cite{lloyd1982least}. 
It is an iterative refinement method that alternates on two steps. The first step is to assign each data point to its nearest cluster center. 
The second step is to update each cluster center, as the mean of data points assigned to that cluster in the first step. The algorithm terminates when the data assignments no longer change. When the data points are well separated, Lloyd's algorithm often converges in a few iterations and achieves good performance~\cite{ostrovsky2013effectiveness}, but its complexity can grow superpolynomially in the worst case \cite{arthur2006slow}. Despite of many variants of Lloyd's algorithm to improve its clustering performance and computational complexity in general cases, we focus on solving the $k$-means clustering problem using the original Lloyd's heuristic due to its simplicity.




\subsection{Secure Federated $k$-means}
We consider solving a $k$-means clustering problem, over a federated learning system that consists of a central server and $n$ distributed clients. The entire dataset consists of private data points arbitrarily distributed across the clients. For each $j \in [n]$, we denote the set of local data points at client $j$ as ${\cal D}_j$, $j \in [n]$. Specifically, we would like to run the Lloyd's algorithm to obtain a $k$-clustering of the clients' data, with the coordination of the server.

A similar problem was recently considered in~\cite{Dennis21a}, where a federated clustering protocol, $k$-FED, was proposed to approximately execute Lloyd's heuristic with one-shot communication from the clients to the server. The basic idea of $k$-FED is to have each client $j$ run Lloyd to generate a $k_j$-clustering of its local data ${\cal D}_j$, and send the cluster centers to the server, and then the server runs Lloyd again to cluster the centers into $k$ groups, generating a $k$-clustering of the entire dataset. We note that $k$-FED is limited in the following two aspects: 1) to run \kfed, each client $j$ needs to know the actual number of local clusters $k_j$, and the clustering performance is guaranteed only when the number of local clusters is less than $\sqrt{k}$, and the number of data points in each non-empty local cluster is large enough, which may hardly hold in practical scenarios; 2) the clients' private data is leaked to the server through the client centers directly sent to the server, which are computed as the means of the data points in respective local clusters.


Motivated by the above limitations, we aim to design a federated $k$-means clustering algorithm that simultaneously achieves the following requirements.
\begin{itemize}[leftmargin=*]
\item {\bf Universal performance:} No loss of clustering performance compared with the centralized Lloyd's algorithm on the entire dataset, regardless of the data distribution across the clients;
\item {\bf Data $t$-privacy:} We consider a threat model where the clients and the server are \emph{honest-but-curious}. For some security parameter $t < n$, during the protocol execution, no information about a client's private data should be leaked, even when any subset of up to $t$ clients collude with each other. Further, the server should not know explicitly any private data of the clients, as well as the locations of the cluster centers.
\end{itemize}





We propose Secure Federated Clustering (or \scheme) that simultaneously satisfies the above performance and security requirements. In \scheme, private data is secret shared across clients such that information-theoretical privacy is guaranteed against up to $t$ colluding clients. Using computations over coded data, the server decodes pair-wise distances between the data points and the centers, which are sufficient to advance Lloyd iterations. In this way, \scheme exactly recovers the center updates of centralized Lloyd's algorithm, with the server knowing nothing about the data points and the centers. 




\section{The Proposed \scheme Protocol}



As \scheme leverages cryptographic primitives to protect data privacy, which operates on finite fields, we assume that 
the elements of each data point are from a finite field $\mathbb{F}_q$ of order $q$. In practice, one can quantize each element of the data points onto $\mathbb{F}_q$, for some sufficiently large prime $q$, such that overflow does not occur to any computation during the execution of the Lloyd's algorithm. See Appendix~\ref{appendix1} for one such quantization method. 

\subsection{Protocol Description}\label{sec:protocol}

{\bf Secure data sharing.}
In the first phase of \scheme, each client mixes its local data with random noises to generate secret shares of the data, and communicates the shares with the other clients. 

For each $j \in [n]$, and each ${\bf x}_i \in {\cal D}_j$, client $j$ first evenly partitions ${\bf x}_i \in \mathbb{F}_q^d$ into $\ell$ segments ${\bf x}_{i,1},\ldots,{\bf x}_{i,\ell} \in \mathbb{F}_q^{\frac{d}{\ell}}$, for some design parameter $\ell$. Next, client $j$ randomly samples $t$ noise vectors ${\bf z}_{i,\ell+1},\ldots, {\bf z}_{i,\ell+t}$ from $\mathbb{F}_q^{\frac{d}{\ell}}$. Then, for a set of distinct parameters $\{\beta_1,\ldots,\beta_{\ell+t}\}$ from $\mathbb{F}_q$ that are agreed upon among all clients and the server, following the data encoding scheme of Lagrange Coded Computing (LCC) \cite{yu2019lagrange}, client $j$ performs Lagrange interpolation on the points $(\beta_1, {\bf x}_{i,1}),\ldots, (\beta_{\ell}, {\bf x}_{i,\ell}),(\beta_{\ell+1}, {\bf z}_{i,\ell+1}),\ldots,(\beta_{\ell+t}, {\bf z}_{i,\ell+t})$ to obtain the following polynomial
\begin{equation}
\label{encoding2}
{\bf f}_i(\alpha)=\sum\limits_{u=1}^{\ell}{{\bf x}}_{i,u} \cdot \prod_{v\in[\ell+t]\backslash\{u\}}\frac{\alpha-\beta_{v}}{\beta_{u}-\beta_{v}}
+\sum\limits_{u=\ell+1}^{\ell+t}{{\bf z}}_{i,u} \cdot \prod_{v\in[\ell+t]\backslash\{u\}}\frac{\alpha-\beta_{v}}{\beta_{u}-\beta_{v}}.
\end{equation}
Note that ${\bf f}_i(\beta_u) = {\bf x}_{i,u}$ for $u=1,\ldots,\ell$, and ${\bf f}_i(\beta_u) = {\bf z}_{i,u}$ for $u=\ell+1,\ldots,\ell+t$.

Finally, for another set of public parameters $\{\alpha_1,\ldots,\alpha_n\}$ that are pair-wise distinct, and $\{\beta_1,\ldots,\beta_{\ell+t}\} \cap \{\alpha_1,\ldots,\alpha_n\} = \emptyset$, client $j$ evaluates ${\bf f}_i(\alpha)$ at $\alpha_{j'}$ to compute a secret share $\widetilde{\mathbf{x}}_{i,j'}$ of ${\bf x}_i$ for client $j'$, for all $j' \in [n]$ (i.e., $\widetilde{\mathbf{x}}_{i,j'} = {\bf f}_i(\alpha_{j'})$), and sends $\widetilde{\mathbf{x}}_{i,j'}$ to client $j'$.

By the end of secure data sharing, each client $j$ possesses a secret share of the entire dataset ${\bf X}$, i.e., $\{\widetilde{\mathbf{x} }_{i,j}\}_{i=1}^m$.


{\bf Coded center update.} 
The proposed \scheme runs in an iterative manner until the centers converge. The server starts by randomly sampling a $k$-clustering assignment $\mathcal{S}$. In the subsequent iterations, the computations of the distances from the data points to the centers, and the center updates are all performed on the secret shares of the original data.

In each iteration, given the current clustering $({\cal S}_1,\ldots,{\cal S}_k)$, each client $j$ updates the secret shares of the centers as $\widetilde{\boldsymbol{\mu}}_{h,j}= \sum_{i \in \mathcal{S}_h}\widetilde{\mathbf{x}}_{i,j}$, for all $h \in [k]$. Then, for each $h \in [k]$ and each $i \in [m]$, client $j$ computes the \emph{coded} distance from data point $i$ to center $h$ as $\widetilde{d}_{i,h,j} = \left \| \widetilde{\boldsymbol{\mu}}_{h,j}-\left |  \mathcal{S}_h \right | \cdot \widetilde{\mathbf{x}}_{i,j} \right \|_2^2$, and sends it to the server. For each $(i,h) \in [m] \times [k]$, the coded distance $\widetilde{d}_{i,h,j}$ can be viewed as the evaluation of a polynomial $\phi_{i,h}(\alpha) = \left \| \sum_{i' \in \mathcal{S}_h} {\bf f}_{i'}(\alpha)- |\mathcal{S}_h| \cdot {\bf f}_i(\alpha) \right \|_2^2$ at point $\alpha_j$. As $\phi_{i,h}(\alpha)$ has degree $2(\ell+t-1)$, the server can interpolate $\phi_{i,h}(\alpha)$ from the computation results of $2\ell+2t-1$ clients.
Having recovered $\phi_{i,h}(\alpha)$, the server computes $d'_{i,h} = \sum_{u=1}^{\ell} \phi_{i,h}(\beta_u) = \sum_{u=1}^{\ell} \left \| \sum_{i' \in \mathcal{S}_h} {\bf x}_{i',u}- |\mathcal{S}_h| \cdot {\bf x}_{i,u} \right \|_2^2$, and normalizes it by the square of the size of cluster $h$ to obtain the distance from data point $i$ to cluster center $h$, i.e., $d_{i,h} = \left \| \boldsymbol{\mu}_{h} - \mathbf{x}_i \right \|^2_2 = \frac{d'_{i,h}}{|{\cal S}_h|^2}$.\footnote{Here the division is performed in the real field.}
Having obtained the distances between all centers and all data points, the server updates the clustering assignment by first identifying the nearest center for each data point, and assigning the data point to the corresponding cluster. We summarize the proposed \scheme protocol in Algorithm~\ref{alg1}.


\begin{algorithm}[h]
\caption{\scheme}
\label{alg1}
\KwData{Datasets ${\cal D}_1,\ldots,{\cal D}_n$ on $n$ clients}
\KwIn{Number of clusters $k$}
\KwOut{$k$-clustering $\mathcal{S} = \left( \mathcal{S}_1,\ldots,\mathcal{S}_k  \right ) $}
\emph{Phase 1: Secure data sharing}\;
\For{$j\leftarrow 1$ \KwTo $n$}{
\lForEach{${\bf x}_i \in D_j$}{Client $j$ generates a secret share $\widetilde{\mathbf{x}}_{i,j'}$, and sends it to client $j'$, $\forall j' \in [n]$}
}
\emph{Phase 2: Coded center update}\;
Server randomly samples an initial clustering $\mathcal{S} $ and broadcasts to all clients\;

\While{$\mathcal{S} = \left( \mathcal{S}_1,\ldots,\mathcal{S}_k  \right )$ is different from last iteration}{
\For{$j\leftarrow 1$ \KwTo $n$}{
\lForEach{$h \in [k]$}{Client $j$ updates local coded center $\widetilde{\boldsymbol{\mu}}_{h,j}= \sum_{i \in \mathcal{S}_h}\widetilde{\mathbf{x}}_{i,j}$}
\For{$(i,h) \in [m] \times [k]$}{Client $j$ computes a coded distance $\widetilde{d}_{i,h,j} = \left \| \widetilde{\boldsymbol{\mu}}_{h,j}-\left |  \mathcal{S}_h \right | \cdot \widetilde{\mathbf{x}}_{i,j} \right \|_2^2$, and sends it to the server}
}



\For{$(i,h) \in [m] \times [k]$}{Server recovers $d'_{i,h}=\left \| \boldsymbol{\mu}_{h}-\left |  \mathcal{S}_h \right | \cdot \mathbf{x}_{i} \right \|_2^2$ from $ (\widetilde{d}_{i,h,1},\ldots,\widetilde{d}_{i,h,n})$ via interpolating and evaluating a polynomial of degree $2t+2\ell-2$\;

Server computes the actual distance $d_{i,h} = \frac{  d'_{i,h} }{\left | \mathcal{S}_h \right | ^2} $ from ${\bf x}_i$ to $\boldsymbol{\mu}_{h}$ \;
}



Server associates each data point to its nearest center and updates the clustering ${\cal S}$.
}
\end{algorithm}
\vspace{-2mm}


\subsection{Theoretical Analysis}\label{sec:theoretical}
We theoretically analyze the performance, privacy, and complexity of the proposed \scheme protocol.

\begin{theorem}
\label{main_thm}
For the federated clustering problem with $n$ clients, the proposed \scheme scheme with parameter $\ell$ achieves universal performance and data $t$-privacy, when $2t+2\ell -1 \le n$.
\end{theorem}

\begin{proof}
The data $t$-privacy of \scheme against colluding clients follows from the $t$-privacy of LCC encoding. As shown in Theorem~1 of \cite{yu2019lagrange}, each data point ${\bf x}_i$ is information-theoretical security against $t$ colluding clients. In other words, for any subset ${\cal C} \subset [n]$ of size not larger than $t$, the mutual information $I({\bf x}_i;(\widetilde{{\bf x}}_{i,j})_{j \in {\cal C}}) = 0$.

As explained in the above protocol description, for a given clustering $({\cal S}_1, \ldots, {\cal S}_k)$, as long as $n \ge 2 \ell+2t-1$, the server is able to exactly reconstruct the polynomial $\phi_{i,h}(\alpha)$ from the computation results of the $n$ clients, and subsequently compute the exact distance $d_{i,h}$ from data point ${\bf x}_i$ to the cluster center $\boldsymbol{\mu}_{h} = \frac{\sum_{{\bf x}_{i'} \in {\cal S}_h} {\bf x}_{i'}}{|{\cal S}_h|}$. Therefore, the clustering update of \scheme exactly follows that of a centralized Lloyd's algorithm, and achieves identical clustering performance. It is easy to see that this performance is achieved universally regardless of the clients' knowledge about the numbers of local clusters, and the distribution of data points.

Finally, with the computation results received from the clients, the server is only able to decode the distances between the data points and the current cluster centers, but not their individual values. This achieves data privacy against a curious server. 
\end{proof}

\begin{remark}
Taking advantage of a key property of the Lloyd's algorithm where only the distances between points and centers are needed to carry out the center update step, \scheme is able to protect clients' data privacy against the server, without sacrificing any clustering performance.
\end{remark}

\begin{remark}\label{rmk:performance}
\scheme achieves the same clustering performance as the original Lloyd's algorithm, which is nevertheless not guaranteed to produce the optimal clustering, and its performance relies heavily on the clustering initialization~\cite{celebi2013comparative}. It was proposed in~\cite{Awasthi12} to improve center initialization by first projecting the dataset onto the subspace spanned by its top $k$ singular vectors, and applying an approximation algorithm (see, e.g.,~\cite{Kanungo03,har2004coresets}) on the projected dataset to obtain the initial centers. How one can incorporate these initialization tricks into \scheme without compromising data privacy remains a compelling research question.
\end{remark}

Theorem \ref{main_thm} reveals a trade-off between privacy and efficiency for \scheme. As a larger $t$ allows for data privacy against more curious colluding clients, a larger $\ell$ reduces both the communication loads among the clients, and the computation complexities at the clients and the server.

{\bf Complexity analysis.}
We analyze the communication and computation complexities of \scheme. Let $m_j = |{\cal D}_j|$ be the number local data points at client $j$, and $s$ be the number of iterations. 


\textit{Communication complexity of client $j$.} The communication load of client $j$ consists of two parts: 1) The communication cost of client $j$ for data sharing is $\mathcal{O}(\frac{m_jdn}{\ell})$; 2) In each iteration, client $j$ communicates $km$ computation results to the server, incurring a cumulative communication cost of $\mathcal{O}(kms)$. 
The total communication cost of client $j$ during the execution of \scheme is $\mathcal{O}(\frac{m_jdn}{\ell} + kms)$.

\textit{Computation complexity of client $j$.} The computation performed by client $j$ consists of two parts: 1) Utilizing fast polynomial interpolation and evaluation~\cite{kedlaya2011fast}, client $j$ can generate the secret shares of its local data with complexity 
$\mathcal{O}(\frac{m_j d}{\ell} n \log^{2}n)$; 
2) Within each iteration, client $j$ needs to first compute a secret share for each of the $k$ centers, and then compute the distance from each data point to each center in the coded domain. This takes a total of $\mathcal{O}(\frac{kmds}{\ell})$ operations. Hence, the total computation complexity of client $j$ is $\mathcal{O}(\frac{m_j d}{\ell} n \log^{2}n + \frac{kmds}{\ell})$.


\textit{Computation complexity of the server.} The computation occurs on server in each iteration consists of two parts. 1) The decoding and recovery of actual pair-wise distances take $\mathcal{O}(km(\ell+t)\log^{2}(\ell+t))$ operations; 2) The cost of cluster assignment on each data point is $\mathcal{O}(k)$. Thus, the total computation complexity of the server is $\mathcal{O}(kms(\ell+t)\log^{2}(\ell+t))$.


\section{Experiments}





In this section, we empirically evaluate the performance and complexity of the proposed \scheme protocol, for federated clustering of synthetic and real datasets. We compare \scheme with centralized implementation of Lloyd’s heuristic and the $k$-FED protocol as baselines. All experiments are carried out on a machine using Intel(R) Xeon(R) Gold 5118 CPU @ 2.30GHz, with 12 cores of 48 threads.  



The \emph{center separation} technique~\cite{Awasthi12} is applied to all three algorithms, for clustering initialization to improve the performance. It is a one-time method to generate $k$ groups of data ${\cal S}'_1, \ldots, {\cal S}'_k$ from a given clustering assignment. Given the current cluster centers $\boldsymbol{\mu}_{1}, \ldots, \boldsymbol{\mu}_{k} $, each $S'_h$ is selected such that $S'_{h} \leftarrow\left\{i:\left\| {\bf x}_i-\boldsymbol{\mu}_{h} \right\|_{2} \leqslant \frac{1}{3}\left\|{\bf x}_i-\boldsymbol{\mu}_{s} \right\|_{2}, \forall s \in [k]\right\}$. In this way, the data in ${\cal S}'_h$ is close to center $\boldsymbol{\mu}_{h}$ and far from all other centers. Then, the means of ${\cal S}'_1, \ldots, {\cal S}'_k$ are used as the initial centers for the subsequent iterative execution of the algorithms.

We cluster data points with ground truth labels, and use \emph{classification accuracy} as the performance measure. For each $i \in [m]$, let $y_i$ and $y'_i$ respectively denote the ground truth label of data point ${\bf x}_i$, and the index of the cluster it belongs to after the clustering process, such that $y_i,y'_i \in [k] $. As the clustering operation does not directly predict ${\bf x}_i$'s label, the classification of ${\bf x}_i$ is subject to a potential permutation $\pi$ of the label set $[k]$. 
We use the Hungarian method \cite{kuhn1955hungarian} to find the optimal permutation $\pi_0$ that maximizes $ \sum_{i \in [m]} \mathbbm{1}\left ( y_i =\pi(y'_i)  \right)$, where $\mathbbm{1}(\cdot)$ is the indicator function. Then, the accuracy of a $k$-means clustering algorithm is given by $ \frac{1}{m}\sum_{i \in [m]} \mathbbm{1}\left ( y_i = \pi_0(y'_i)  \right ) $.

\subsection{Separating Mixture of Gaussians}
We first consider a synthetic dataset, where the data is generated by a mixture of $k$ Gaussians. That is, we generate isotropic Gaussian clusters comprised of a total of $m$ data points each with $d$ features. Each data point $\mathbf{x}_i$ is sampled from an average of normal distributions with density $p(\mathbf{x})=\sum_{h=1}^{k} \frac{1}{k}  \mathcal{N}\left(\boldsymbol{\mu}_h, \boldsymbol{\Sigma}_h \right)$. 
The centers $\{\boldsymbol{\mu}_h\}_{h=1}^k$ are randomly chosen. The covariance matrix $\bold{\Sigma}_h$ is diagonal with covariance $\sigma_h^2$. 
The standard deviations are the same for all $h$, i.e., $\sigma_h = \sigma$, $ h \in [k]$. The ground truth label of each data point 
is given by the index of the closest center to that particular point. 


To compare \scheme with $k$-FED, we generate clients' datasets with different $k'$, where the parameter $k'$ is the maximum number of clusters a client can have data points in and indicates the level of data heterogeneity across clients. The $k$-FED scheme assumes that this number is known to the clients \cite{Dennis21a}, which is not required for our \scheme. We generate local datasets for the cases of $k' \leq \sqrt{k}$, which is required by \kfed to have performance guarantees, and also for the case of $k'=k$, which corresponds to i.i.d. data distribution across clients.


We consider two Gaussian mixtures of dimension $d=100$, with standard deviation $\sigma=1$ and $\sigma=20$ respectively. For each value of $\sigma$, we consider a setting of $k=4$ clusters, $n=10$ clients, and $m=10000$ data points are generated; and another setting of $k=16$ clusters, $n=16$ clients, and $m=16384$ data points are generated. For each of the datasets, the centralized Lloyd, \kfed, and \scheme are executed for $10$ runs, and their performance are shown in Table~\ref{table:accuracy1} and Table~\ref{table:accuracy20}.

\begin{table}[h]
\centering
\caption{Clustering accuracy (\%) for a Gaussian mixture with $\sigma=1$.}
\label{table:accuracy1}
\resizebox{\columnwidth}{!}{%
\begin{tabular}{|c|ccc|ccc|}
\hline
                  & \multicolumn{3}{c|}{$k = 4$} & \multicolumn{3}{c|}{$k = 16$} \\ \cline{2-7} 
                  &$k'=1$ &$k'=2$ &$k'=4$ &$k'=2$ &$k'=4$ &$k'=16$ \\ \hline
Centralized Lloyd &$100.0\pm 0.0 $ &$100.0\pm 0.0 $ &$100.0\pm 0.0 $ & $88.2\pm 5.5$& $88.2\pm 5.5$& $88.2\pm 5.5$\\
$k$-FED             &$75.0\pm 0.1$&$99.7\pm 2.1$&$25.8\pm 0.6$&$80.3\pm 5.3$&$78.7\pm 3.7$&$7.6\pm 0.1$\\
SecFC             &$100.0\pm 0.0 $&$100.0\pm 0.0 $&$100.0\pm 0.0 $&$86.0\pm 1.5$&$86.0\pm 1.5$&$86.0\pm 1.5$\\
\hline
\end{tabular}%
}
\end{table}



For each value of $\sigma$, and each combination of $k$ and $k'$, \scheme consistently achieves almost identical performance with that of the centralized Lloyd, which is better than the accuracy of \kfed. For a fixed $k$, the accuracy of \scheme almost stays the same for different values of $k'$, demonstrating its universal superiority. In sharp contrast, the performance of \kfed varies significantly with $k'$, and drops drastically for the case of $k'=k$. With a larger $\sigma$, the Gaussian clusters are more scattered, hence increasing the difficulty for clustering algorithms to correctly classify data points. This is empirically verified by the performance drop when increasing $\sigma$ from $1$ to $20$, for all three algorithms.

\begin{table}[t]
\centering
\caption{Clustering accuracy (\%) for a Gaussian mixture with $\sigma=20$.}
\label{table:accuracy20}
\resizebox{\columnwidth}{!}{%
\begin{tabular}{|c|ccc|ccc|}
\hline
                  & \multicolumn{3}{c|}{$k = 4$} & \multicolumn{3}{c|}{$k = 16$} \\ \cline{2-7} 
                  &$k'=1$ &$k'=2$ &$k'=4$ &$k'=2$ &$k'=4$ &$k'=16$ \\ \hline
Centralized Lloyd &$96.3\pm 0.1$ &$96.3\pm 0.1$ &$96.3\pm 0.1$ &$84.0\pm 0.2$&$84.0\pm 0.2$&$84.0\pm 0.2$\\
$k$-FED             &$46.7\pm 0.0$&$86.6\pm 5.5$&$25.6\pm 0.1$&$42.0\pm 2.2$&$41.7\pm 1.2$&$8.4\pm 0.1$\\
SecFC             &$96.3\pm 0.0$&$96.3\pm 0.0$&$96.3\pm 0.0$&$84.0\pm 0.1$&$84.0\pm 0.1$&$84.0\pm 0.1$\\\hline
\end{tabular}%
}
\vspace{-5mm}
\end{table}



{\bf Complexity evaluation.} 
We also evaluate the execution overheads of \scheme, under different combinations of system parameters. By default, the clustering is done for $k=4$ clusters on $n=10$ clients, for a dataset generated by parameters $m=1000$, $d=20$, and $\sigma=1$. For \scheme, as the computations at both the clients and the server are independent across centers and data points, we parallelize the computations onto different processes to speed up the execution of \scheme. Based on the capacity of the system's computing power, the distance computations of the $m$ data points at each client are parallelized onto multiple processes, and the distance decodings of the $mk$ (data, center) pairs at the server are also parallelized onto multiple processes.
We present the run times of \kfed and \scheme, as functions of the parameters $n$, $d$, and $m$ respectively in Figure~\ref{fig:runtime}. Here we measure the average run time of \scheme spent in each client-server iteration, and the overall run time of \kfed as it performs one-shot communication from clients to the server. We omit the communication times in all scenarios as they are negligible compared with computation times. 



We observe in Figure~\ref{fig:runtime} that, as expected, \scheme is slower than \kfed (in fact, one iteration of \scheme is already slower than the entire execution of \kfed). However, the execution time of \scheme in each iteration is rather short ($<16 s$ in the worst case), which demonstrates its potential for practical secure clustering tasks.
For different client number $n$, the security parameter $t$ is kept to be $t = \lceil\frac{1}{3}n\rceil$ to provide privacy against up to $\frac{n}{3}$ colluding clients. The client run time is independent of $n$ because all clients compute on the coded data of the entire dataset in \scheme. Server complexity increases slightly with $n$ as the increased $t$ contributes to a higher decoding complexity. For different data dimension $d$, the server run time remains almost unchanged as it only decodes the distance between data points and centers. The run times of the clients in \scheme and the overall run time of \kfed both increase with $d$ as they are operating on data points with a higher dimension. Finally, when the number of data points $m$ is increased, the run times of server and clients in \scheme, and the overall run time of \kfed all grow almost linearly.

\begin{figure}[h]
\vspace{-3mm}
  \includegraphics[width=\linewidth]{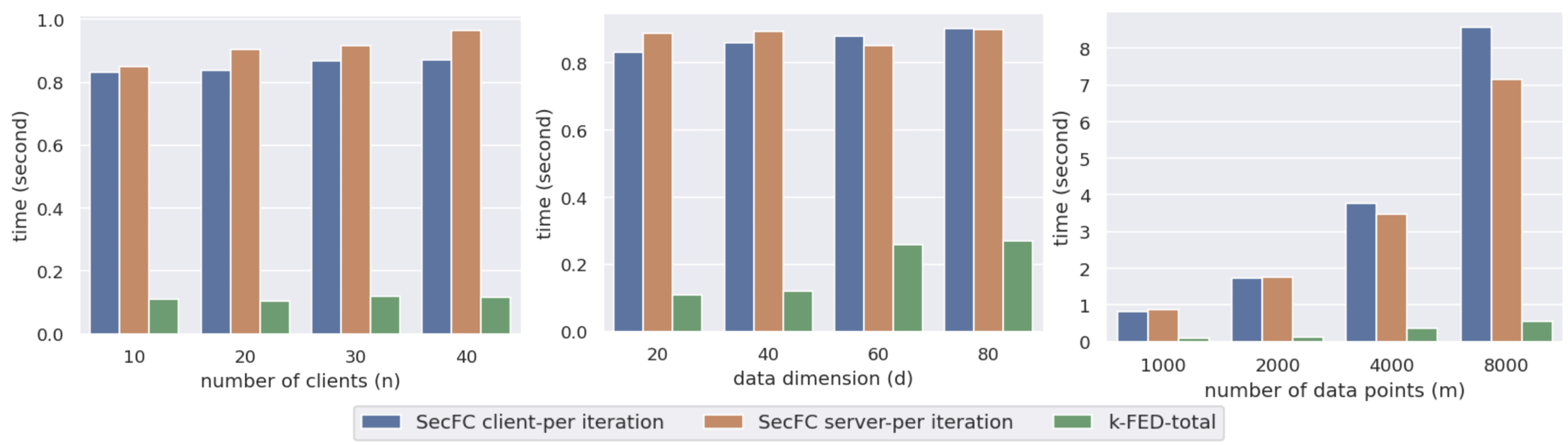}
 \vspace{-3mm} 
\caption{Run time of \scheme and \kfed under different system parameters.}\label{fig:runtime}
\end{figure}


\vspace{-5mm}
\subsection{Experiments on Real Data}
We also evaluate the clustering performance of \scheme on the MNIST handwritten digit dataset~\cite{mnist}. As in \cite{chung2022federated}, we use rotations to construct a clustered dataset such that different rotation angles correspond to data generated by different distributions. For the set of images in MNIST labelled by a certain digit, we rotate the images by a degree of $0, 90,180,270$ respectively to construct a dataset with $k=4$ clusters. We execute \scheme and the two baseline algorithms on the datasets generated for digits $2$ and $3$ over $n=10$ clients, and present their performance in Table~\ref{table:mnistrotate}.


\begin{table}[h]
\centering
\caption{Clustering accuracy (\%) for Rotated MNIST datasets.}
\label{table:mnistrotate}
\resizebox{\columnwidth}{!}{%
\begin{tabular}{|c|ccc|ccc|}
\hline
                  & \multicolumn{3}{c|}{Digit $2$} & \multicolumn{3}{c|}{Digit $3$} \\ \hline
$k=4$               &$k'=1$ &$k'=2$ &$k'=4$       &$k'=1$ &$k'=2$ &$k'=4$        \\ \hline
Centralized Lloyd &$99.5\pm 0.1$&$99.5\pm 0.1$&$99.5\pm 0.1$& $98.1\pm 0.0$&$98.1\pm 0.0$&$98.1\pm 0.0$\\
$k$-FED             &$50.5\pm 0.0$ &$90.3\pm 10.4$&$25.8\pm 0.1$&$97.7\pm 0.0$&      $98.1\pm0.0$   &$68.2\pm 9.4$\\
SecFC             &$99.5\pm 0.1$&$99.5\pm 0.1$&$99.5\pm 0.1$&$98.1\pm 0.0$&$98.1\pm 0.0$&$98.1\pm 0.0$\\\hline
\end{tabular}%
}
\vspace{-5mm}
\end{table}

Similar to the synthetic data, \scheme achieves almost identical performance with centralized Lloyd in all tasks of clustering digit $2$ and digit $3$ images, which is insensitive to the value of $k'$. \kfed achieves a lower accuracy in general, which also varies substantially with $k'$.

\vspace{-3mm}
\section{Strengthening \scheme with Membership Privacy}
\vspace{-2mm}
One shortcoming of \scheme is that distribution of the global dataset, i.e., $({\cal I}_1,\ldots,{\cal I}_n)$, where ${\cal I}_j$ denotes the set of indices of the local data points at client $j$, will become publicly known during data sharing. It would be desirable for client $j$ to also keep ${\cal I}_j$ private, so that from the final clustering result $({\cal S}_1,\ldots,{\cal S}_k)$, no client or the server can infer which data point belongs to which client. 

{\bf Private ID alignment.} This can be achieved by adding an additional step of private ID alignment before data sharing in \scheme. We assume that each data point has a unique ID (e.g., its hash value), and we denote the set of IDs of the data points in ${\cal D}_j$ at client $j$ as ${\cal E}_j$. Before the clustering task starts, both ${\cal D}_j$ and ${\cal E}_j$ are kept private at client $j$, and are not known by the other clients and the server. As the first step of the clustering algorithm, all clients collectively run a private set union (PSU) protocol (see, e.g.,~\cite{psu:1,psu:2,psu:3}), over the local ID sets ${\cal E}_1,\ldots,{\cal E}_n$, such that by the end of PSU, each client knows the set of all data IDs ${\cal E} = \cup_{j\in[n]}\mathcal{E}_n$ in the network, without knowing the individual ID sets of other clients. Let $|{\cal E}| =m$. All the clients and the server agree on a bijective map between ${\cal E}$ and $[m]$. This also induces a bijective map between ${\cal E}_j$ and ${\cal I}_j$ at each client $j$.

After private ID alignment, each client learns the total number of data points in the system. During the secure data sharing phase, to protect the privacy of ${\cal I}_j \subset [m]$, each client $j$ secret shares a data $\overline{{\bf x}}_i^{(j)}$ for \emph{each} $i=1,\ldots,m$, as described in Section~\ref{sec:protocol}. Here $\overline{{\bf x}}_i^{(j)} = {\bf x}_i$ for $i \in {\cal I}_j$ and $\overline{{\bf x}}_i^{(j)} = {\bf 0}$ otherwise. We denote the secret share of $\overline{{\bf x}}_i^{(j')}$ created by client $j'$ for client $j$ as $\widetilde{{\bf x}}^{(j')}_{i,j}$. Unlike the original data sharing where each client $j$ only receives one secret share $\widetilde{{\bf x}}_{i,j}$ from the owner of ${\bf x}_i$, for each $i=1,\ldots,m$, now with membership privacy, client $j$ receives $n$ shares $\widetilde{{\bf x}}^{(1)}_{i,j},\ldots,\widetilde{{\bf x}}^{(n)}_{i,j}$, one from each client. Each client $j$ computes $\widetilde{{\bf x}}_{i,j} = \sum_{j'=1}^n \widetilde{{\bf x}}^{(j')}_{i,j}$. This completes the data sharing phase. Next, the clients perform the same coded center update process as in \scheme to obtain the final $k$-clustering. 

With the extension of incorporating the private ID alignment step, \scheme can further protect the membership privacy of the data points, without sacrificing data privacy and clustering performance. We give more detailed description and analysis of this extended version of \scheme in Appendix~\ref{sec:membership}.
 
\vspace{-2mm}
\section{Conclusion and discussions}\label{sec:conclusion}
\vspace{-2mm}
We propose a secure federated clustering algorithm named \scheme, for clustering data points arbitrarily distributed across a federated learning network. \scheme's design builds upon the well-known Lloyd's algorithm, and \scheme universally achieves identical performance to a centralized Lloyd's execution. Moreover, \scheme protects the data privacy of each client via secure data sharing, and computation over coded data throughout the clustering process. It guarantees information-theoretic privacy against colluding clients, and no explicit knowledge about the data points and cluster centers at the server. Experiment results from clustering synthetic and real data demonstrate the universal superiority of \scheme and its computational practicality under a wide range of system parameters. Finally, we propose an extended version of \scheme to further provide membership privacy for the clients' datasets, such that, while knowing the final clustering, no party knows the owner of any data point. 

While the high level of security provided by \scheme is desirable for many applications, it cannot be directly applied to scenarios where the centers of clusters  (e.g., in mean estimation~\cite{lai2016agnostic,cuesta2008robust}), or the membership of specific data points (e.g., in clustering-based outlier detection~\cite{loureiro2004outlier,jiang2008clustering}) need to be known explicitly. These problems can be resolved respectively by having $\ell+t$ clients to collectively reveal the cluster centers, and having each client publish the IDs of their local data. 

\bibliographystyle{unsrt}
\bibliography{IEEEabrv,lib_v1}

\begin{thebibliography}{10}

\bibitem{mcmahan2017communication}
Brendan McMahan, Eider Moore, Daniel Ramage, Seth Hampson, and Blaise~Aguera
  y~Arcas.
\newblock Communication-efficient learning of deep networks from decentralized
  data.
\newblock In {\em Artificial intelligence and statistics}, pages 1273--1282.
  PMLR, 2017.

\bibitem{yang2018applied}
Timothy Yang, Galen Andrew, Hubert Eichner, Haicheng Sun, Wei Li, Nicholas
  Kong, Daniel Ramage, and Fran{\c{c}}oise Beaufays.
\newblock Applied federated learning: Improving {Google} keyboard query
  suggestions.
\newblock {\em arXiv preprint arXiv:1812.02903}, 2018.

\bibitem{jalalirad2019simple}
Amir Jalalirad, Marco Scavuzzo, Catalin Capota, and Michael Sprague.
\newblock A simple and efficient federated recommender system.
\newblock In {\em IEEE/ACM International Conference on Big Data Computing,
  Applications and Technologies}, pages 53--58, 2019.

\bibitem{rieke2020future}
Nicola Rieke, Jonny Hancox, Wenqi Li, Fausto Milletari, Holger~R. Roth, Shadi
  Albarqouni, Spyridon Bakas, Mathieu~N. Galtier, Bennett~A. Landman, Klaus
  Maier-Hein, et~al.
\newblock The future of digital health with federated learning.
\newblock {\em NPJ digital medicine}, 3(1):1--7, 2020.

\bibitem{t2020personalized}
Canh T.~Dinh, Nguyen Tran, and Josh Nguyen.
\newblock Personalized federated learning with {Moreau} envelopes.
\newblock {\em Advances in Neural Information Processing Systems},
  33:21394--21405, 2020.

\bibitem{fallah2020personalized}
Alireza Fallah, Aryan Mokhtari, and Asuman Ozdaglar.
\newblock Personalized federated learning: A meta-learning approach.
\newblock {\em arXiv preprint arXiv:2002.07948}, 2020.

\bibitem{li2021ditto}
Tian Li, Shengyuan Hu, Ahmad Beirami, and Virginia Smith.
\newblock Ditto: Fair and robust federated learning through personalization.
\newblock In {\em International Conference on Machine Learning}, pages
  6357--6368. PMLR, 2021.

\bibitem{mansour2020three}
Yishay Mansour, Mehryar Mohri, Jae Ro, and Ananda~Theertha Suresh.
\newblock Three approaches for personalization with applications to federated
  learning.
\newblock {\em arXiv preprint arXiv:2002.10619}, 2020.

\bibitem{ghosh2020efficient}
Avishek Ghosh, Jichan Chung, Dong Yin, and Kannan Ramchandran.
\newblock An efficient framework for clustered federated learning.
\newblock {\em Advances in Neural Information Processing Systems},
  33:19586--19597, 2020.

\bibitem{ouyang2021clusterfl}
Xiaomin Ouyang, Zhiyuan Xie, Jiayu Zhou, Jianwei Huang, and Guoliang Xing.
\newblock Cluster{FL}: a similarity-aware federated learning system for human
  activity recognition.
\newblock In {\em International Conference on Mobile Systems, Applications, and
  Services}, pages 54--66, 2021.

\bibitem{kim2021dynamic}
Yeongwoo Kim, Ezeddin Al~Hakim, Johan Haraldson, Henrik Eriksson, Jos{\'e}
  Mairton~B. da~Silva, and Carlo Fischione.
\newblock Dynamic clustering in federated learning.
\newblock In {\em ICC 2021-IEEE International Conference on Communications},
  pages 1--6. IEEE, 2021.

\bibitem{chung2022federated}
Jichan Chung, Kangwook Lee, and Kannan Ramchandran.
\newblock Federated unsupervised clustering with generative models.
\newblock In {\em AAAI 2022 International Workshop on Trustable, Verifiable and
  Auditable Federated Learning}, 2022.

\bibitem{Dennis21a}
Don~Kurian Dennis, Tian Li, and Virginia Smith.
\newblock Heterogeneity for the win: One-shot federated clustering.
\newblock In {\em International Conference on Machine Learning}, pages
  2611--2620. PMLR, July 2021.

\bibitem{lloyd1982least}
Stuart Lloyd.
\newblock Least squares quantization in {PCM}.
\newblock {\em IEEE Transactions on Information Theory}, 28(2):129--137, 1982.

\bibitem{zhu2019deep}
Ligeng Zhu, Zhijian Liu, and Song Han.
\newblock Deep leakage from gradients.
\newblock {\em Advances in Neural Information Processing Systems}, 32, 2019.

\bibitem{geiping2020inverting}
Jonas Geiping, Hartmut Bauermeister, Hannah Dr{\"o}ge, and Michael Moeller.
\newblock Inverting gradients-how easy is it to break privacy in federated
  learning?
\newblock {\em Advances in Neural Information Processing Systems},
  33:16937--16947, 2020.

\bibitem{wang2019beyond}
Zhibo Wang, Mengkai Song, Zhifei Zhang, Yang Song, Qian Wang, and Hairong Qi.
\newblock Beyond inferring class representatives: User-level privacy leakage
  from federated learning.
\newblock In {\em IEEE INFOCOM 2019-IEEE Conference on Computer
  Communications}, pages 2512--2520. IEEE, 2019.

\bibitem{fowl2021robbing}
Liam Fowl, Jonas Geiping, Wojtek Czaja, Micah Goldblum, and Tom Goldstein.
\newblock Robbing the {Fed}: Directly obtaining private data in federated
  learning with modified models.
\newblock {\em arXiv preprint arXiv:2110.13057}, 2021.

\bibitem{yu2019lagrange}
Qian Yu, Songze Li, Netanel Raviv, Seyed Mohammadreza~Mousavi Kalan, Mahdi
  Soltanolkotabi, and Salman~A. Avestimehr.
\newblock Lagrange coded computing: Optimal design for resiliency, security,
  and privacy.
\newblock In {\em The 22nd International Conference on Artificial Intelligence
  and Statistics}, pages 1215--1225. PMLR, 2019.

\bibitem{psu:1}
Jae~Hong Seo, Jung~Hee Cheon, and Jonathan Katz.
\newblock Constant-round multi-party private set union using reversed laurent
  series.
\newblock In {\em International Workshop on Public Key Cryptography}. Springer,
  2012.

\bibitem{psu:2}
Sivakanth Gopi, Pankaj Gulhane, Janardhan Kulkarni, Judy~Hanwen Shen, Milad
  Shokouhi, and Sergey Yekhanin.
\newblock Differentially private set union.
\newblock In {\em International Conference on Machine Learning}. PMLR, 2020.

\bibitem{psu:3}
Keith Frikken.
\newblock Privacy-preserving set union.
\newblock In Jonathan Katz and Moti Yung, editors, {\em Applied Cryptography
  and Network Security}. Springer, 2007.

\bibitem{Kanungo03}
Tapas Kanungo, David~M. Mount, Nathan~S. Netanyahu, Christine~D. Piatko, Ruth
  Silverman, and Angela~Y. Wu.
\newblock A local search approximation algorithm for $k$-means clustering.
\newblock In {\em Annual Symposium on Computational Geometry}, June 2002.

\bibitem{Awasthi12}
Pranjal Awasthi and Or~Sheffet.
\newblock Improved spectral-norm bounds for clustering.
\newblock In {\em Approximation, Randomization, and Combinatorial Optimization.
  Algorithms and Techniques}, pages 37--49. Springer, 2012.

\bibitem{Dhillon99}
Inderjit~S. Dhillon and Dharmendra~S. Modha.
\newblock A data-clustering algorithm on distributed memory multiprocessors.
\newblock In {\em Workshop on Large-Scale Parallel KDD Systems, SIGKDD}. August
  1999.

\bibitem{Joshi03}
Manasi~N. Joshi.
\newblock Parallel {$K$-Means} algorithm on distributed memory multiprocessors.
\newblock {\em Technical report, University of Minnesota}, 2003.

\bibitem{Xu99}
Xiaowei Xu, Jochen J{\"a}ger, and Hans-Peter Kriegel.
\newblock A fast parallel clustering algorithm for large spatial databases.
\newblock In {\em High performance data mining}, pages 263--290. Springer,
  1999.

\bibitem{Kargupta01}
Hillol Kargupta, Weiyun Huang, Krishnamoorthy Sivakumar, and Erik Johnson.
\newblock Distributed clustering using collective principal component analysis.
\newblock {\em Knowledge and Information Systems}, 3(4):422--448, 2001.

\bibitem{Januzaj03}
Eshref Januzaj, Hans-Peter Kriegel, and Martin Pfeifle.
\newblock Towards effective and efficient distributed clustering.
\newblock In {\em Workshop on Clustering Large Data Sets}, 2003.

\bibitem{Balcan13}
Maria-Florina~F. Balcan, Steven Ehrlich, and Yingyu Liang.
\newblock Distributed $ k $-means and $ k $-median clustering on general
  topologies.
\newblock {\em Advances in Neural Information Processing Systems}, 26, 2013.

\bibitem{Smith17}
Virginia Smith, Chao-Kai Chiang, Maziar Sanjabi, and Ameet~S Talwalkar.
\newblock Federated multi-task learning.
\newblock {\em Advances in neural information processing systems}, 30, 2017.

\bibitem{Sattler20}
Felix Sattler, Klaus-Robert M{\"u}ller, and Wojciech Samek.
\newblock Clustered federated learning: Model-agnostic distributed multitask
  optimization under privacy constraints.
\newblock {\em IEEE Transactions on Neural Networks and Learning Systems},
  32(8):3710--3722, 2020.

\bibitem{Balakrishnan21}
Ravikumar Balakrishnan, Tian Li, Tianyi Zhou, Nageen Himayat, Virginia Smith,
  and Jeff Bilmes.
\newblock Diverse client selection for federated learning via submodular
  maximization.
\newblock In {\em International Conference on Learning Representations}, 2022.

\bibitem{Mukherjee19}
Sudipto Mukherjee, Himanshu Asnani, Eugene Lin, and Sreeram Kannan.
\newblock Clustergan: Latent space clustering in generative adversarial
  networks.
\newblock In {\em AAAI Conference on Artificial Intelligence}, volume~33, pages
  4610--4617, 2019.

\bibitem{Liu20}
Steven Liu, Tongzhou Wang, David Bau, Jun-Yan Zhu, and Antonio Torralba.
\newblock Diverse image generation via self-conditioned gans.
\newblock In {\em IEEE/CVF Conference on Computer Vision and Pattern
  Recognition}, pages 14286--14295, 2020.

\bibitem{Vaidya03}
Jaideep Vaidya and Chris Clifton.
\newblock Privacy-preserving $k$-means clustering over vertically partitioned
  data.
\newblock In {\em ACM SIGKDD International Conference on Knowledge Discovery
  and Data Mining}, pages 206--215, 2003.

\bibitem{Jagannathan05}
Geetha Jagannathan and Rebecca~N Wright.
\newblock Privacy-preserving distributed $k$-means clustering over arbitrarily
  partitioned data.
\newblock In {\em ACM SIGKDD International Conference on Knowledge Discovery in
  Data Mining}, pages 593--599, 2005.

\bibitem{Bunn07}
Paul Bunn and Rafail Ostrovsky.
\newblock Secure two-party $k$-means clustering.
\newblock In {\em ACM Conference on Computer and Communications Security},
  pages 486--497, 2007.

\bibitem{Patel12}
Sankita Patel, Sweta Garasia, and Devesh Jinwala.
\newblock An efficient approach for privacy preserving distributed $k$-means
  clustering based on shamir’s secret sharing scheme.
\newblock In {\em IFIP International Conference on Trust Management}, pages
  129--141. Springer, 2012.

\bibitem{Yuan17}
Jiawei Yuan and Yifan Tian.
\newblock Practical privacy-preserving mapreduce based $k$-means clustering
  over large-scale dataset.
\newblock {\em IEEE Transactions on Cloud Computing}, 7(2):568--579, 2017.

\bibitem{Mohassel19}
Payman Mohassel, Mike Rosulek, and Ni~Trieu.
\newblock Practical privacy-preserving $k$-means clustering.
\newblock {\em Cryptology ePrint Archive}, 2019.

\bibitem{ostrovsky2013effectiveness}
Rafail Ostrovsky, Yuval Rabani, Leonard~J. Schulman, and Chaitanya Swamy.
\newblock The effectiveness of lloyd-type methods for the $k$-means problem.
\newblock {\em Journal of the ACM}, 59(6):1--22, 2013.

\bibitem{arthur2006slow}
David Arthur and Sergei Vassilvitskii.
\newblock How slow is the $k$-means method?
\newblock In {\em Symposium on Computational Geometry}, pages 144--153, 2006.

\bibitem{celebi2013comparative}
M.~Emre Celebi, Hassan~A. Kingravi, and Patricio~A. Vela.
\newblock A comparative study of efficient initialization methods for the
  $k$-means clustering algorithm.
\newblock {\em Expert Systems with Applications}, 40(1):200--210, 2013.

\bibitem{har2004coresets}
Sariel Har-Peled and Soham Mazumdar.
\newblock On coresets for $k$-means and $k$-median clustering.
\newblock In {\em ACM Symposium on Theory of Computing}, pages 291--300, 2004.

\bibitem{kedlaya2011fast}
Kiran~S. Kedlaya and Christopher Umans.
\newblock Fast polynomial factorization and modular composition.
\newblock {\em SIAM Journal on Computing}, 40(6):1767--1802, 2011.

\bibitem{kuhn1955hungarian}
Harold~W Kuhn.
\newblock The {Hungarian} method for the assignment problem.
\newblock {\em Naval Research Logistics Quarterly}, 2(1-2):83--97, 1955.

\bibitem{mnist}
Yann LeCun, L{\'e}on Bottou, Yoshua Bengio, and Patrick Haffner.
\newblock Gradient-based learning applied to document recognition.
\newblock {\em Proceedings of the IEEE}, 86(11):2278--2324, 1998.

\bibitem{lai2016agnostic}
Kevin~A Lai, Anup~B Rao, and Santosh Vempala.
\newblock Agnostic estimation of mean and covariance.
\newblock In {\em 2016 IEEE 57th Annual Symposium on Foundations of Computer
  Science (FOCS)}, pages 665--674. IEEE, 2016.

\bibitem{cuesta2008robust}
JA~Cuesta-Albertos, C~Matr{\'a}n, and A~Mayo-Iscar.
\newblock Robust estimation in the normal mixture model based on robust
  clustering.
\newblock {\em Journal of the Royal Statistical Society: Series B (Statistical
  Methodology)}, 70(4):779--802, 2008.

\bibitem{loureiro2004outlier}
Antonio Loureiro, Luis Torgo, and Carlos Soares.
\newblock Outlier detection using clustering methods: a data cleaning
  application.
\newblock In {\em Proceedings of KDNet Symposium on Knowledge-based systems for
  the Public Sector}. Springer Bonn, 2004.

\bibitem{jiang2008clustering}
Sheng-yi Jiang and Qing-bo An.
\newblock Clustering-based outlier detection method.
\newblock In {\em 2008 Fifth international conference on fuzzy systems and
  knowledge discovery}, volume~2, pages 429--433. IEEE, 2008.

\end{thebibliography}

\newpage
\appendix
\section*{Appendix}

\section{Data Quantization}
\label{appendix1}

In \scheme, the operations of data sharing and distance computation are carried out in a finite field $\mathbb{F}_q$, for some large prime $q$. Hence, for a data point ${\bf x}_i$ from the real field, one needs to first quantize it onto $\mathbb{F}_q$.
Specifically, for some scaling factor $\lambda$ that determines the quantization precision, we first scale ${\bf x}_i$ by $\lambda$, and embed the scaled value onto $\mathbb{F}_q$ as follows.
\begin{equation}
\label{quantization}
\bar{\mathbf{x}}_i = \mathcal{Q}(\mathbf{x}_i,\lambda)=\left\{\begin{array}{@{}ll}
\lfloor \lambda \mathbf{x}_i \rfloor ,   &\mathrm{if}~\mathbf{x}_i\geq 0  \\
\lfloor q+\lambda \mathbf{x}_i \rfloor, & \mathrm{if}~\mathbf{x}_i<0
\end{array}\right.,
\end{equation}
where $\lfloor x\rfloor$ denotes the largest integer less than or equal to $x$, and the quantization function ${\cal Q}$ is applied element-wise. Assuming each element of $\lambda {\bf x}_i$ is within $[-\eta, \eta)$ for some $\eta > 0$, then on the range of ${\cal Q}$, $\left [ 0,\eta \right )$ is the image of the positive part, and $\left [ q-\eta,q \right )$ is the image of the negative part. While $\lambda$ controls the precision loss of quantization, it also needs to be chosen so that overflow does not occur during computations of \scheme. A larger $\eta$ requires a larger finite field size $q$ to avoid computation overflow, and a smaller $\eta$ leads to a higher precision loss between $\mathbf{x}_i$ and $\bar{\mathbf{x}}_i$. We can choose a proper scaling factor $\lambda$ to preserve enough precision while keeping the field size practical. 

To avoid computation overflow, we should choose $q$ such that all the intermediate computation results on the scaled data $\lambda {\bf x}_i$ are within the range $\left ( -\frac{q}{2},\frac{q}{2} \right ) $.
In the worst case, the largest distance across data points and cluster centers  $D \triangleq \underset{i \in [m], h \in [k]}{\max} \left \| \boldsymbol{\mu}_{h}-\left |  \mathcal{S}_h \right | \cdot \mathbf{x}_{i} \right \|_2^2 $ results in the largest output value.
Therefore, we should choose $q$ that is at least $2\lambda^2 D$. 

\section{Extension of \scheme with Membership Privacy} \label{sec:membership}

After the step of private ID alignment, each client learns the total number of data points in the system. Each client $j$ knows the set of IDs of its private data points as ${\cal I}_j$, but does not know the sets of data IDs of the other clients. 

During the secure data sharing phase, to protect the privacy of ${\cal I}_j \subset [m]$, each client $j$ first creates a data point $\overline{{\bf x}}_i^{(j)}$ for \emph{each} $i=1,\ldots,m$, where \begin{equation}\label{eq:auxiliary}
    \overline{{\bf x}}_i^{(j)} = \begin{cases}
    {\bf x}_i, & i \in {\cal I}_j\\{\bf 0}, & \textup{otherwise} 
    \end{cases}.
\end{equation}

Next, for each $i \in [m]$, and some design parameter $\ell$, as described in Section~\ref{sec:protocol}, using Lagrange interpolation and  evaluation, each client $j'$ encodes $\overline{{\bf x}}_i^{(j')}$, and creates a secret share $\widetilde{{\bf x}}^{(j')}_{i,j}$ for each $j \in [n]$ as
\begin{equation}
\label{encoding-membership}
\widetilde{{\bf x}}^{(j')}_{i,j}=\sum\limits_{u=1}^{\ell}\overline{{\bf x}}_{i,u}^{(j')} \cdot \prod_{v\in[\ell+t]\backslash\{u\}}\frac{\alpha_j-\beta_{v}}{\beta_{u}-\beta_{v}}
+\sum\limits_{u=\ell+1}^{\ell+t}{{\bf z}}_{i,u}^{(j')} \cdot \prod_{v\in[\ell+t]\backslash\{u\}}\frac{\alpha_j-\beta_{v}}{\beta_{u}-\beta_{v}},
\end{equation}
where $\overline{{\bf x}}_{i,1}^{(j')},\ldots,\overline{{\bf x}}_{i,\ell}^{(j')} \in \mathbb{F}_{q}^{\frac{d}{\ell}}$ are segments of $\overline{{\bf x}}_i^{(j')}$, and ${{\bf z}}_{i,\ell+1}^{(j')},\ldots,{{\bf z}}_{i,\ell+t}^{(j')} \in \mathbb{F}_q^{\frac{d}{\ell}}$ are noises randomly generated by client $j'$ for each $i \in [m]$.

Unlike the original data sharing of \scheme where each client $j$ only receives one secret share $\widetilde{{\bf x}}_{i,j}$ from the owner of ${\bf x}_i$, for each $i \in [m]$, client $j$ receives $n$ shares $\widetilde{{\bf x}}^{(1)}_{i,j},\ldots,\widetilde{{\bf x}}^{(n)}_{i,j}$, one from each client. Each client $j$ computes $\widetilde{{\bf x}}_{i,j} = \sum_{j'=1}^n \widetilde{{\bf x}}^{(j')}_{i,j}$. We can see from $(\ref{eq:auxiliary})$ and $(\ref{encoding-membership})$ that
\begin{align}
    \widetilde{{\bf x}}_{i,j} &= \sum\limits_{u=1}^{\ell}
    \left(\sum\limits_{j'=1}^{n}\overline{{\bf x}}_{i,u}^{(j')}\right) \cdot \prod_{v\in[\ell+t]\backslash\{u\}}\frac{\alpha_j-\beta_{v}}{\beta_{u}-\beta_{v}}
+\sum\limits_{u=\ell+1}^{\ell+t}\left(\sum\limits_{j'=1}^{n}{{\bf z}}_{i,u}^{(j')}\right) \cdot \prod_{v\in[\ell+t]\backslash\{u\}}\frac{\alpha_j-\beta_{v}}{\beta_{u}-\beta_{v}} \nonumber \\
&= \sum\limits_{u=1}^{\ell}
    {\bf x}_{i,u}\cdot \prod_{v\in[\ell+t]\backslash\{u\}}\frac{\alpha_j-\beta_{v}}{\beta_{u}-\beta_{v}}
+\sum\limits_{u=\ell+1}^{\ell+t}\overline{{\bf z}}_{i,u} \cdot \prod_{v\in[\ell+t]\backslash\{u\}}\frac{\alpha_j-\beta_{v}}{\beta_{u}-\beta_{v}}, \label{coded_data}
\end{align}
where $\overline{{\bf z}}_{i,u} = \sum\limits_{j'=1}^{n}{{\bf z}}_{i,u}^{(j')}$ is still uniformly distributed in $\mathbb{F}_q^{\frac{d}{\ell}}$.

This completes the phase of secure data sharing. 

As the data share in (\ref{coded_data}) resembles the algebraic structure of ${\bf f}_i(\alpha_j)$ in (\ref{encoding2}), for all $i \in [m]$ and $j \in [n]$, we can carry out the same coded center update process as in the original \scheme to reach the final $k$-clustering. 

The data $t$-privacy of this extended version of \scheme is guaranteed by the Lagrange secure data sharing in (\ref{encoding-membership}). Also, as mentioned above, it achieves the same clustering performance as the original \scheme, which demonstrates its universal performance according to Theorem~\ref{main_thm}. 

\end{document}